\documentclass{article}
\usepackage{ijcai20}

\usepackage{times}
\usepackage{soul}
\usepackage{url}
\usepackage[hidelinks]{hyperref}
\usepackage[utf8]{inputenc}
\usepackage[small]{caption}
\usepackage{graphicx}
\usepackage{amsmath}
\usepackage{amsthm}
\usepackage{booktabs}
\usepackage{algorithm}
\usepackage{algorithmic}
\usepackage{subcaption}

\newtheorem{theorem}{Theorem}

\newcommand{\KLD}[2]{D_{\mathrm{KL}} \left( \left. \left. #1 \right|\right| #2 \right) }

\title{Learning Reusable Options for Multi-Task Reinforcement Learning}

\author{
Francisco M. Garcia$^{1,2}$
\and%
Chris Nota$^2$\And
Philip S. Thomas$^2$
\affiliations
$^1$Amazon Alexa\\
$^2$University of Massachusetts Amherst
\emails
\{fmgarcia,cnota,psthomas\}@cs.umass.edu
}
\begin{document}

\maketitle

\begin{abstract}
  Reinforcement learning (RL) has become an increasingly active area of research in recent years. Although there are many algorithms that allow an agent to solve tasks efficiently, they often ignore the possibility that prior experience related to the task at hand might be available. For many practical applications, it might be unfeasible for an agent to learn how to solve a task from scratch, given that it is generally a computationally expensive process; however, prior experience could be leveraged to make these problems tractable in practice. In this paper, we propose a framework for exploiting existing experience by learning reusable \emph{options}. We show that after an agent learns policies for solving a small number of problems, we are able to use the trajectories generated from those policies to learn reusable options that allow an agent to quickly learn how to solve novel and related problems.
\end{abstract}

\section{Introduction}

Reinforcement learning (RL) techniques have experienced much of their success in simulated environments, such as video games \cite{atari} or board games \cite{go,tdgammon}. One of the main reasons why RL has worked so well in these applications is that we are able simulate millions of interactions with the environment in a relatively short period of time, allowing the agent to experience a large number of different situations in the environment and learn the consequences of its actions. 

In many real world applications, however, where the agent interacts with the physical world, it might not be easy to generate such a large number of interactions. The time and cost associated with training such systems could render RL an unfeasible approach for training in large scale.

As a concrete example, consider training a large number of humanoid robots (agents) to move quickly, as in the Robocup competition \cite{peter_stone_run}. Although the agents have similar dynamics, subtle variations mean that a single policy shared across all agents would not be an effective solution. Furthermore, learning a policy from scratch for each agent is too data-inefficient to be practical. %unlikely to yield a good solution, as the problem it is too difficult to solve. 
As shown by Farchy et al. (2013), this type of problem can be addressed by leveraging the experience obtained from solving a related task (e.g., walking) to quickly learn a policy for each individual agent that is tailored to a new task (e.g., running).

The situation where agents might need to solve many related, but unique, tasks also occurs in industry; an example would be robots (agents) tasked with sorting items in fulfillment centers. A simple approach, like using PD controllers, would fail to adapt to the forces generated from picking up objects with different weight distributions, causing the agent to drop the objects. RL is able to mitigate this problem by learning a policy for each agent that is able to make corrections quickly, which is tailored to the robot's dynamics. However, training a new policy for each agent would be far too costly to be a practical solution.

In these scenarios, it is possible to use a small number of policies learned by a subset of the agents, and then leverage the experience obtained from learning those policies to allow the remaining agents to quickly learn their corresponding policies. This approach can turn problems that are prohibitively expensive to solve into relatively simple problems.

To make use of prior experience and improve learning on new related problems in RL, several lines of work, which are complementary to each other, have been proposed and are actively being studied. \emph{Transfer learning}, \cite{tl_rl} refers to the problem of adapting information acquired while solving one task to another. One might consider learning a mapping function that allows for a policy learned in \emph{one} task to be used in a different task, \cite{quadrotor}, or simply learn a mapping of the value function learned in one task to another, \cite{qtransfer}. These techniques can be quite effective, but are also limited in that they consider mapping information from a \emph{source} task to a \emph{target} task; that is, they do not learn a general transfer strategy for many related tasks.

Another approach to reusing prior knowledge is through \emph{meta learning} or learning to learn \cite{schmid1,schmid2}. In the context of RL, the goal under this framework is usually for an agent to be exposed to a number of tasks where it can learn some general behavior that generalizes to new tasks. For example, Finn et al. (2017), showed that an agent who learns how to walk forward is able to find a general policy that can quickly be adapted to learn to walk backwards \cite{maml}.   

One last technique to leverage prior experience, and the one this paper focuses on, is through temporally extended actions or \emph{temporal abstractions} \cite{macro_tech,smdp}. While in the standard RL framework the agent has access to a set of primitive actions (i.e., actions that last for one time step), temporally extended actions allow an agent to execute actions that last for several time-steps. They introduce a bias in the behavior of the agent which, if appropriate for the problem at hand, results in dramatic improvements in how quickly the agent learns to solve a new task compared to only using primitive actions \cite{macro_tech}. 

A popular representation for temporally extended actions is the \emph{options} framework (formally introduced in the next section), which is the focus of this work. It has been shown that options learned in a specific task or set of tasks, can be reused to improve learning on new tasks \cite{eigen_option,option_critic}; however, this often requires knowledge from the user about which options or how many options are appropriate for the type of problems the agent will face.

In this paper, we propose learning reusable options for a set of related tasks with minimal information provided by the user. We consider the scenario where the agent must solve a large numbers of tasks and show that after learning a well-performing policy for a small number of problems, we can learn an appropriate number of options that facilitates learning in a remaining set of tasks. Ideally the trajectories used to learn options would be obtained from optimal policies, but for many learning algorithms it cannot be guaranteed that the learned policies are actually optimal.
We propose learning a set of options that minimize the expected number of decisions needed to represent trajectories generated from the policies learned by the agent for a small number of problems, while also maximizing the probability of generating those trajectories. Our experiments show that after learning to solve a small number of tasks, the learned options allow the agent to much more quickly solve the remaining tasks.

\section{Background and Notation}
\label{sec:background}

A \textit{Markov decision process} (MDP) is a tuple, $M = (\mathcal S, \mathcal A,P,R, \gamma, d_0)$, where $\mathcal S$ is the set of possible states of the environment, $\mathcal A$ is the set of possible actions that the agent can take, $P(s,a,s')$ is the probability that the environment will transition to state $s'\in \mathcal S$ if the agent executes action $a \in \mathcal A$ in state $s \in \mathcal S$, $R(s,a,s')$ is the expected reward received after taking action $a$ in state $s$ and transitioning to state $s'$, $d_0$ is the initial state distribution, and $\gamma \in [0,1]$ is a discount factor for rewards received in the future. We use $t$ to index the time-step and write $S_t$, $A_t$, and $R_t$ to denote the state, action, and reward at time $t$. A \textit{policy}, $\pi: \mathcal S \times \mathcal A \to [0,1]$, provides a conditional distribution over actions given each possible state: $\pi(s,a)=\Pr(A_t=a|S_t=s)$. We denote a trajectory of length $t$ as $h_t = (s_0,a_0,r_0, \dots, s_{t-1},a_{t-1},r_{t-1},s_t)$, that is, $h_t$ is defined as a sequence of states, actions and rewards observed after following some policy for $t$ time-steps. 
%$H_t = (S_0, A_0, R_0, \dots, S_{t-1}, A_{t-1}, R_{t-1}, S_t)$.
% 

This work focuses on learning \emph{temporally extended} actions---actions lasting for multiple time-steps---that can be used for a set of related tasks. We consider the setting where an agent must solve a set of related tasks, where each task is an MDP, $M = (\mathcal S, \mathcal A, P_M , R_M, \gamma, d^M_0)$; that is, each task is an MDP with its own transition function, reward function and initial state distribution, with shared state and action sets. Specifically, our work focuses on learning reusable \emph{options} \cite{Sutton98intra-optionlearning,smdp} for a set of related tasks.

An option, $o=(\mathcal I_o, \mu_o, \beta_o)$, is a tuple in which $\mathcal I_o \subseteq \mathcal S$ is the set of states in which option $o$ can be executed (the \emph{initiation set}), $\mu_o$ is a policy that governs the behavior of the agent while executing $o$, and $\beta_o: \mathcal S \to [0,1]$ is a termination function that determines the probability that $o$ terminates in a given state. We assume that $\mathcal I_o = \mathcal S$ for all options $o$; that is, the options are available at every state. 
The options framework does not dictate how an agent should choose between available options or how options should be discovered.
A common approach to selecting between options is to a learn a \emph{policy over options}, which is defined by the probability of choosing an option in a particular state.
Two recent popular approaches to option discovery are eigenoptions \cite{eigen_option} and the option-critic architecture \cite{option_critic}.

The \emph{eigenoptions} \cite{eigen_option} of an MDP are the optimal policies for a set of implicitly defined reward functions called \emph{eigenpurposes}.
Eigenpurposes are defined in terms of \emph{proto-value functions} \cite{pvfs}, which are in turn derived from the eigenvectors of a modified adjacency matrix over states for the MDP.
The intuition is that no matter the true reward function, the eigenoptions allow an agent to quickly traverse the transition graph, resulting in better exploration of the state space and faster learning. However, there are two major downsides: 1) the adjacency matrix is often not known \emph{a priori}, and may be difficult or impossible to construct for large MDPs, and 2) for each eigenpurpose, constructing the corresponding eigenoption requires solving a new MDP.

The option-critic architecture \cite{option_critic} is a more direct approach that learns options and a policy over options simultaneously.
The option policies and their termination functions are trained using policy gradient methods, while the policy over options may be trained using any technique.
One issue that often arises within this framework is that the termination functions of the learned options tend to collapse to ``always terminate''.
In a later publication, the authors built on this work to consider the case where there is a cost associated with switching options \cite{option_critic_penalty}. This method resulted in the agent learning to use a single option while it was appropriate and terminate when an option switch was needed, allowing it to discover improved policies for a particular task. 
The authors argue that minimizing the use of the policy over options may be desirable, as the cost of choosing an option may be greater than the cost of choosing a primitive action when using an option---e.g., when a planner is used to select an option.
Work recently presented by Harutyunyan et al. (2019) approaches the aforementioned termination problem by explicitly optimizing the termination function of options to focus on small regions of the state space.  
However, while all of these methods can be effective in learning a policy for the task at hand, they do not explicitly take into consideration that the agent might face related, but different, tasks in the future. In contrast, our method discovers options that are useful for a variety tasks.

We build on the idea that minimizing the number of decisions an agent must make will lead to the discovery of generally useful temporal abstractions, and propose an offline technique where options are learned after solving a small number of tasks. The options can then be leveraged to quickly solve new related problems the agent will face in the future. We use the trajectories generated by the agent when learning policies for a small number of problems, and learn an appropriate set of options by directly minimizing the expected number of decisions the agent makes while simultaneously maximizing the probability of generating the observed trajectories. 

\section{Learning Reusable Options from Experience}
\label{sec:problem_statement}

In this section, we formally introduce the objective we use to learn a set of options that are reusable for a set of related tasks. Our algorithm introduces one option at a time until introducing a new option does not improve the objective further. This procedure results in a natural way of learning an adequate number of options without having to pre-define it; a new option is included only if it is able to improve the probability of generating optimal behavior while minimizing the number of decisions made by the agent.

\subsection{Problem Formulation}

In the options framework, at each time-step, $t$, the agent chooses an action, $A_t$, based on the current option, $O_t$. Let $T_t$ be a Bernoulli random variable, where $T_t = 1$ if the previous option, $O_{t-1}$, terminated at time $t$, and $T_t = 0$ otherwise.
If $T_t = 1$, $O_t$ is chosen using the policy over options, $\pi$.
If $T_t = 0$, then the previous option continues, that is, $O_t = O_{t-1}$.
To ensure that we can represent any trajectory, we consider primitive actions to be options which always select one specific action and then terminate; that is, for an option, $o$, corresponding to a primitive, $a$, for all $s \in \mathcal S$, the termination function would be given by $\beta_o(s)=1$, and the policy by $\mu(s,a')=1$ if $a'=a$ and $0$ otherwise.

Let $\mathcal O = \mathcal{O_A} \cup \mathcal{O_O}$ denote a set of options, $\{o_1,\dots,o_n\}$, where $\mathcal{O_A}$ refers to the set of options corresponding to primitive actions and $\mathcal{O_O}$ to the set of options corresponding to temporal abstractions. Furthermore, let $H$ be a random variable denoting a trajectory of length $|H|$ generated by a well-performing policy, and let $H_t$ be a random variable denoting the sub-trajectory of $H$ up to the state encountered at time-step $t$. 
We seek to find a set, $\mathcal O^* = \{o_1^*, \dots, o_n^*$\}, that maximizes the following objective: 

\small
\begin{equation}
\begin{aligned}
J(\pi, \mathcal{O}) = \mathbf{E}\big[ \sum_{t=1}^{|H|} \Pr(T_t=0,H_t|\pi, \mathcal{O}) + \lambda_1 g(H,\mathcal{O_O}) \big], 
\label{eq:objective}
\end{aligned}
\end{equation}
\normalsize

\noindent where $g(h, \mathcal{O_O})$ is a regularizer that encourages a diverse set of options, and $\lambda_1$ is a scalar hyper-parameter.
If we are also free to learn the parameters of $\pi$, then $\mathcal{O}^* \in \underset{\mathcal{O}}{\arg\max} \; \underset{\pi}{\max} \; J(\pi, \mathcal{O})$. 

One choice for $g$ is the average KL divergence on a given trajectory over the set of $m$ options being learned: \linebreak $g(h, \mathcal{O_O}) = \frac{2}{m(m-1)} \sum_{o,o' \in \mathcal{O_O}} \sum_{t=0}^{|h|-1}\KLD{\mu_o(s_t)}{\mu_{o'}(s_t)}$. Note that this term is only defined when we consider two or more options. When that is not the case we set this term to 0.

Intuitively, we seek to find options that terminate as infrequently as possible while still generating well-performing trajectories with high probability.
Notice that minimizing the number of terminations is the same as minimizing the number of decisions made by the policy over options, as each termination requires the policy to subsequently choose a new option.
Given a set of options, a policy over options, and a sample trajectory, we can calculate the joint probability for a trajectory \emph{exactly}. Therefore, we can obtain an accurate estimate of Equation \ref{eq:objective} by averaging over a \emph{set} of sample trajectories. In the next section, we present a slight modification to our objective that results in a practical optimization problem.

\subsection{Optimization Objective for Learning Options}

Given that the agent must learn the corresponding policy for a set of tasks, we can use the experienced gathered from solving a subset of tasks to obtain trajectories demonstrating the optimal behavior learned for these problems. Given a set, $\mathcal H$, of trajectories generated from an initial subset of tasks, we can now estimate the expectation in \eqref{eq:objective} to learn options that can be leveraged in the remaining problems. 

Because the probability of generating any trajectory approaches $0$ as the length of the trajectory increases, we make a slight modification to the original objective that leads to better numerical stability. We explain these modifications after introducing the objective $\hat{J}$, which we optimize in practice:

\begin{equation}
\begin{aligned}
\hat{J}(\pi, \mathcal{O}, \mathcal{H}) =& \frac{1}{\mathcal{H}} \sum_{h \in \mathcal H} \big( \underbrace{\lambda_2 \Pr(H=h |\pi, \mathcal{O})}_\text{probability of generating $h$} \\
&- \underbrace{\frac{ \sum_{t=1}^{|h|} \mathbf{E} \left[ T_t=1 \middle | H_t=h_t, \pi, \mathcal{O} \right] }{|h|}}_\text{expected number of terminations} \\
&+ \underbrace{\lambda_1 g(h, \mathcal{O_O})\big)}_\text{encourage diverse options}.
\label{eq:obj-approx}
\end{aligned}
\end{equation}

The objective in \eqref{eq:obj-approx} is derived from $J$ with the following modifications: 1) the sum of the two first terms replaces a product of two terms obtained from computing the joint probability $\Pr(T_t=0,H)$ in $J$,  2) the summation over terminations for a trajectory (second term) is normalized by the length of the trajectory, and 3) we introduce a scalar weight $\lambda_2$ to balance the contribution of each term to $\hat J$. Although this is not an unbiased estimator of $J$, we derived $\hat J$ from $J$ with the introduction of some modifications for numerical stability. A more detailed discussion on how we arrived to this objective is provided in Appendix A.

We can express \eqref{eq:obj-approx} entirely in terms of the policy over options $\pi$, options $\mathcal{O} = \{o_1,\dots,o_n\}$ and the transition function, $P$. When the transition function is not known, we can estimate it from data by assuming a family of distributions and fitting the parameters.
The following theorems show how to calculate the first two terms in \eqref{eq:obj-approx} from known quantities, allowing us to efficiently maximize the proposed objective.

\begin{theorem}
Given a set of options, $\mathcal{O}$, and a policy, $\pi$, over options, the expected number of terminations for a trajectory $h$ is given by:

\small
\begin{equation}
\begin{aligned}
\sum_{t=1}^{|h|} \mathbf{E} \bigg[ T_t=1  \bigg| H_t=h_t, \pi, \mathcal{O} \bigg] =& \sum_{t=1}^{|h|} \sum_{o \in \mathcal O} \beta_o(s_t)  \\
    % &\times \alpha(s_{t-1},a_{t-1},o,h_{t-1}) \\
    % &\times \big( \sum_{o' \in \mathcal O} \alpha(s_{t-1},a_{t-1},o',h_{t-1}) \big)^{-1} , \nonumber
    &\times \frac{\alpha_{t-1}(o)}{ \sum_{o' \in \mathcal O} \alpha_{t-1}(o')} \bigg) , 
\end{aligned}
\end{equation}
\normalsize

\noindent where we use $\alpha_{t}(o)$ as a shorthand notation for $\mu_o(s_{t},a_{t})  \Pr(O_{t}=o|H_{t}=h_{t},\pi, \mathcal{O})$,

\begin{equation}
\begin{aligned}
\Pr(O_{t}=o|H_{t}=h_{t},\pi, \mathcal{O}) =& \bigg[ \bigg( \pi(s_{t}, o) \beta_o(s_{t}) \bigg) \\
&+  \bigg( P(s_{t-1},a_{t-1},s_{t}) \\ 
&\times \alpha_{t-1}(o) (1-\beta_o(s_{t-1})) \bigg) \bigg], \nonumber
\end{aligned}
\end{equation}

\noindent and $\Pr(O_0=o|H_0=h_0,\pi, \mathcal{O}) = \pi(s_0,o)$.
\end{theorem}

\begin{proof}
See Appendix B.
\end{proof}

\begin{theorem}
Given a set of options $\mathcal{O}$ and a policy $\pi$ over options, the probability of generating a trajectory $h$ of length $|h|$ is given by:

\vspace{-5mm}
\begin{equation}
\begin{aligned}
\Pr(H_{|h|}=h_{|h|} |\pi, \mathcal{O}) =&  d_0(s_0) \big[ \sum_{o \in \mathcal{O}} \pi(s_0,o) \\ \nonumber
&\times \mu_o(s_0,a_0) f(h_{|h|},o,1) \big] \\ \nonumber
&\times \prod_{k=0}^{|h|-1} P(s_k,a_k,s_{k+1}), \nonumber
\end{aligned}
\end{equation}
 where $f$ is a recursive function defined as:

\small
\[
  f(h_t, o, i) = 
  \begin{cases}
    1, & \text{if } i = t \\
    \bigg[ \bigg( \beta_o(s_i) \sum_{o' \in \mathcal{O}} \pi(s_{i+1},o') \\ \times \mu_{o'}(s_{i+1}, a_{i+1}) f(h_t,o',i+1) \bigg)  \\
    + \bigg( (1-\beta_o(s_i)) \mu_o(s_{i+1},a_{i+1}) \\ 
    \times f(h_t,o,i+1) \bigg) \bigg] & \text{otherwise}
  \end{cases}
\]
\normalsize

\end{theorem}

\begin{proof}
See Appendix C.
\end{proof}

Given a parametric representation of the option policies and termination functions for each $o \in \mathcal{O}$ and for the policy $\pi$ over options, we use Theorems $1$ and $2$ to differentiate the objective in \eqref{eq:obj-approx} with respect to their parameters and optimize with any standard optimization technique.

\subsection{Learning Options Incrementally}
\label{sec:problem_statement_algo}

One common issue in option discovery is identifying how many options are needed for a given problem. Oftentimes this number is predefined by the user based on intuition. In such a scenario, one could learn options by simply randomly initializing the parameters of a number of options and optimizing the proposed objective in \eqref{eq:obj-approx}. Instead, we propose not only learning options, but also the number of options needed, by the procedure shown in Algorithm \ref{algo-framework}. This algorithm introduces one option at a time and optimizes the objective $\hat J$ with respect to the policy over options $\pi_{\theta}$, with parameters $\theta$, and the newly introduced option, $o'=(\mu'_{\phi},\beta'_{\psi})$, with parameters $\phi$ and $\psi$, for $N$ epochs. Optimizing both $o'$ and $\pi_{\theta}$ allows us to estimate how much we can improve $\hat J$ given that we keep any previously introduced option fixed.

 After the new option is trained, we measure how much $\hat J$ has improved; if it fails to improve above some threshold, $\Delta$, the procedure terminates. This results in a natural way of obtaining an appropriate number of options, as options stop being added once a new option no longer improves the ability to represent the demonstrated behavior.

\begin{algorithm}[H]
\caption{Option Learning Framework - Pseudocode}
\label{algo-framework}
\begin{algorithmic}[1]

\STATE Collect set of trajectories $\mathcal H$

\STATE Initialize option set $\mathcal O$ with primitive options
\STATE done = false
\STATE $\hat J_{prev} = -\infty$
\WHILE{done == false}
    \STATE Initialize new option $o' = (\mu'_{\phi}, \beta'_{\psi})$, initializing parameters for $\phi$ and $\psi$.
    \STATE $\mathcal O' = \mathcal{O} \cup o'$
    \STATE Initialize parameters $\theta$ of policy $\pi_{\theta}$
    \FOR{k=1,\dots,N}
        \STATE $\hat J_k = \hat J(\pi_{\theta}, \mathcal{O'}, \mathcal{H})$

        \STATE $\theta = \theta + \alpha \frac{\partial \hat 
        J_k}{\partial \theta}$
        
        \STATE $\phi = \phi + \alpha \frac{\partial \hat J_k}{\partial \phi}$
        
        \STATE $\psi = \psi + \alpha \frac{\partial \hat J_k}{\partial \psi}$
        
    \ENDFOR
    
    \IF{$\hat J_N - \hat J_{prev} < \Delta$}
        \STATE done = true
    \ELSE
        \STATE $\mathcal{O} = \mathcal{O'}$
        \STATE $\hat J_{prev} = \hat J_N$
    \ENDIF
\ENDWHILE

\STATE Return new option set $\mathcal{O}$
\end{algorithmic}
\end{algorithm}

\section{Experimental Results}
\label{sec:experiments}

\begin{figure*}[t!]
    \centering
        %     \captionsetup{justification=centering}

        % \captionsetup[subfigure]{justification=centering}
    \begin{subfigure}[t]{0.47\textwidth}
        \centering
        \includegraphics[height=5cm, width=7.0cm]{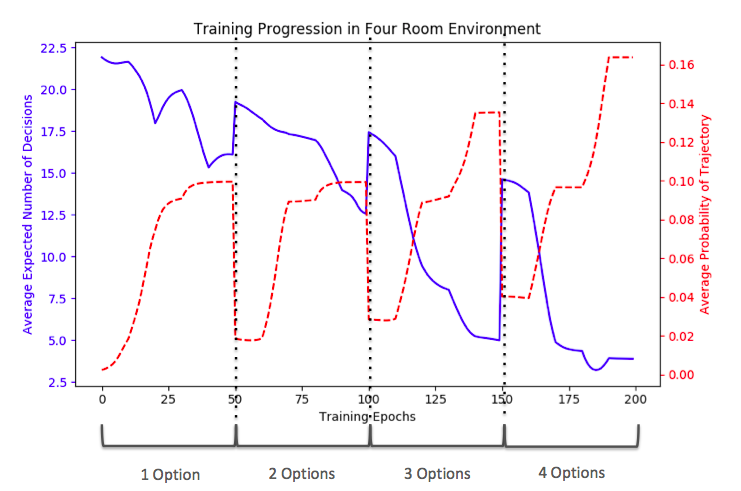}
        \caption{Visualization of loss for sampled trajectories over 200 training epochs. Every 50 training epochs a new option is introduced. For a given set of sampled trajectories, the decreasing average number of decisions made by $\pi$ is shown in blue and the increasing probability of generating the observed trajectories is shown in red.}
        \label{fig:option-loss}
    \end{subfigure}%
    \hspace{1mm} 
    \begin{subfigure}[t]{0.5\textwidth}
        \centering
        \includegraphics[height=5cm, width=7.9cm]{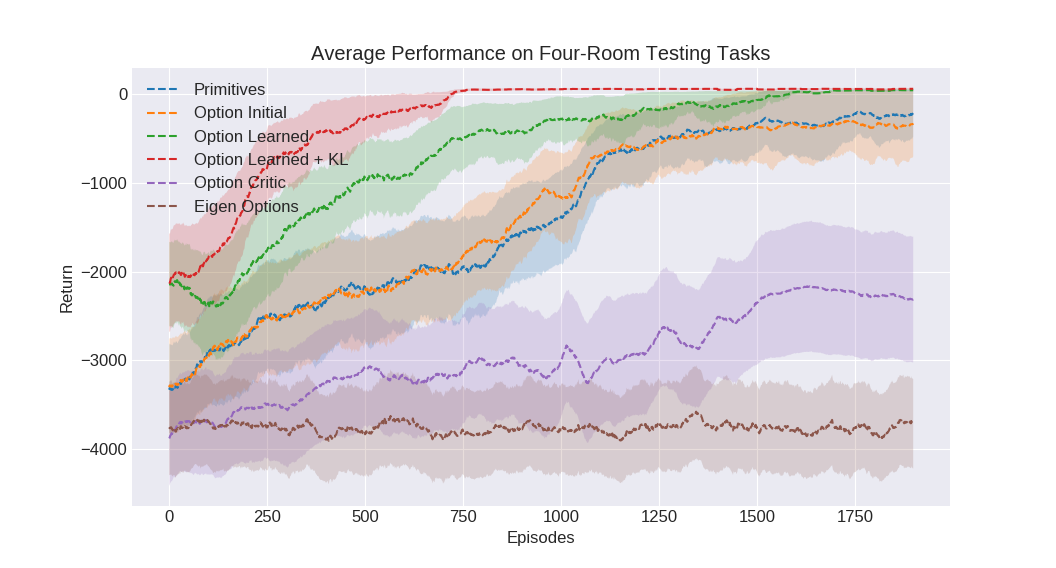}
        \caption{
        Performance Comparison on four rooms domain. Six tasks were used for training and $24$ different tasks for testing. The plot shows the average return (and standard error) on the y-axis as a function of the episode number on the test tasks. For our proposed method we set $\lambda_2 = 100.0$ and $\lambda_1 = 0.001$ when using KL regularization.}
        \label{fig:test-results}
    \end{subfigure}
    \caption{Results on four-room environment showing the improvement of the training loss (left), and learning curves on test problems (right).}
    \label{fig:maze-results}
    \vspace{-3mm}
\end{figure*}

This section describes experiments used to evaluate the proposed offline option learning approach. We show results in the ``four rooms" domain to allow us to visualize and understand the options produced by the approach, and to show empirically that these options produce a clear improvement in learning. We compare against options generated by the option-critic architecture \cite{option_critic} and eigenoptions \cite{eigen_option}. We then extend our experiments to assess the performance of the technique in a few selected problems from the Atari 2600 emulator provided by OpenAI Gym \cite{gym}. These experiments demonstrate that when an agent faces a large number of related tasks, by using the trajectories obtained from solving a small subset of tasks, our approach is able to discover options that significantly improve the learning ability of the agent in the tasks it has yet to solve.

\subsection{Experiments on Four Rooms Environment}

We tested our approach in the four rooms domain: a gridworld of size $40\times40$, in which the agent is placed in a randomly selected start state and needs to reach a randomly selected goal state. At each time-step, the agent executes one of four available actions: moving left, right, up or down, and receives a reward of $-1$. After taking a particular action, the agent moves in the intended direction with probability $0.9$ and in any other direction  with probability $0.1$. Upon reaching the goal state, the agent receives a reward of $+10$. 
We generated $30$ different task variations (by changing the goal and start location) and collected six sample trajectories from optimal policies, learned using Q-learning, from six different start and goal configurations. We evaluated our method on the remaining $24$ tasks.

Each option was represented as a two-layer neural network, with $32$ neurons in each layer, and two output layers: a softmax output layer over the four possible actions representing $\mu$, and a separate sigmoid layer representing $\beta$. We used the tabular form of Q-learning as the learning algorithm with $\epsilon$-greedy exploration.

\begin{figure*}[h]
    \centering
    \includegraphics[width=\linewidth]{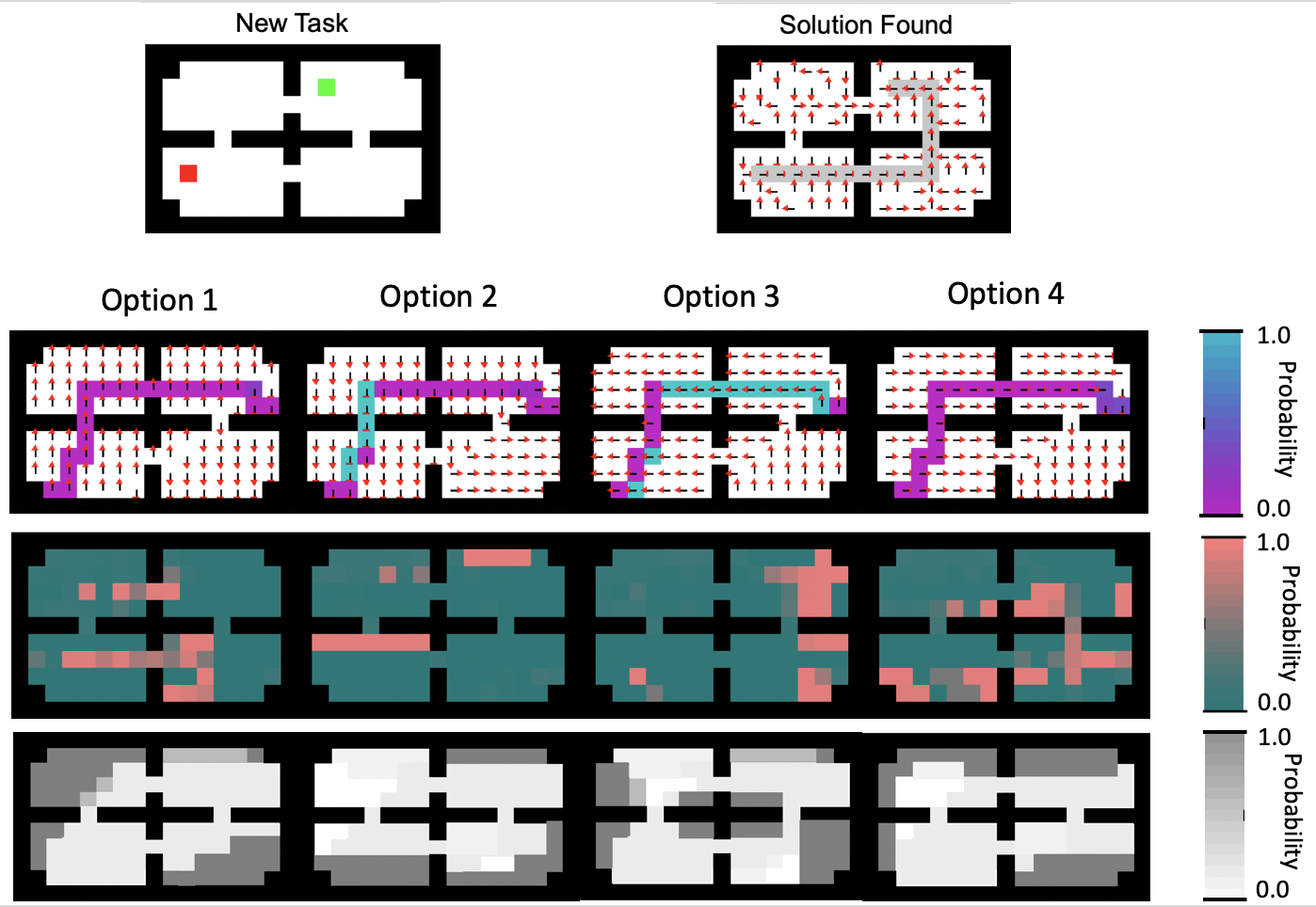}
    \caption{Visualization of our framework for learning options in four rooms domain. A novel task is seen in the top left, where the agent (red) has to navigate to a goal (green). On the top right, we show the solution found by the agent. The three rows below show how the options were learned and exploited in the new task. The highlighted area in the top row show a sample trajectory and the color corresponds to the probability that the option would take the demonstrated action. The middle row shows the probability of each option being executed at each state, while the bottom row shows the corresponding probability of termination.}
    % PHIL EDITED HERE:
%    Notice that this trajectory does not start from the start state of the new task, since the trajectory was generated on a previous task. The arrows show the action that is most likely at each state.
%    Before training (first row), each option is randomly initialized, but after training (second row) each option specializes in a specific skill (a navigation pattern). In this example, the demonstrated trajectory can be generated with high probability by using option $3$ and $2$. The last row shows a heat-map indicating where each options is likely to be called by the policy learned in the new task. As can be seen, the agent learns to use each option in very specific situations. For example, option $1$ is likely to be called to make the agent move up, if it is located in one of the bottom rooms.}
    \label{fig:options}
    \vspace{-3mm}
\end{figure*}

Figure \ref{fig:option-loss} shows the change in the average expected number of terminations and average probability of generating the observed trajectories while learning options, as new options are introduced and adapted to the sampled trajectories. Options were learned over the six sampled optimal trajectories and every $50$ epochs a new option was introduced to the option set, for a total of $4$ options.
For every new option, the change in probability of generating the observed trajectories as well as the change in expected number of decisions reaches a plateau after $30$ or $40$ training epochs. When a new option is introduced, there is a large jump in the loss because a new policy, $\pi$, is initialized arbitrarily to account for the new option set being evaluated. However, after training the new candidate option, the overall loss improves beyond what it was possible before introducing the new option.

In Figure \ref{fig:test-results}, we compare the performance of Q-learning on $24$ novel test tasks (randomly selected start and goal states) using options discovered from offline option learning (with and without regularization using KL divergence), eigenoptions, and option critic. We allowed each competing method to learn options from the same six training tasks and, to ensure a fair comparison, we used the original code provided by the authors. As baselines, we also compare against primitive actions and randomly initialized options.
It might seem surprising that both eigenoptions and the option-critic failed to reach an optimal policy when they were shown to work well in this type of problem; for that we offer the following explanation. Our implementation of four rooms is defined in a much larger state space than the ones where these methods were originally tested, making each individual room much larger. Since the options identified by these methods tend to lead the agent from room to room, it is possible that, once in the correct room, the agent executes an option leading to a different room before it had the opportunity to find the goal. When testing our approach in the smaller version of the four room problem, we found no clear difference to the performance of the competing methods. In this setting, the options learned by our method found an optimal policy in all testing tasks in the allotted number of episodes. We set the threshold $\Delta$ for introducing a new option to $10\%$ of $\hat J$ at the previous iteration and the hyper-parameter $\lambda_2=100.0$. When adding KL regularization, we set $\lambda_1 = 0.001$.

%\begin{figure*}
%    \centering
%    \includegraphics[width=0.9\linewidth]{new_task.png}
%    \vspace{5mm}
%    \includegraphics[width=\linewidth]{options_usage.png}
%    \caption{Visualization of our framework for learning options in four rooms domain. A novel task is seen in the top left, where the agent (red) has to navigate to a goal (green). On the top right, we show the solution found by the agent. The three rows below show how the options were learned and exploited in the new task. The highlighted area in the top two rows show a sample trajectory and the color corresponds to the probability that the option would take the demonstrated action. 
%     PHIL EDITED HERE:
%    Notice that this trajectory does not start from the start state of the new task, since the trajectory was generated on a previous task. The arrows show the action that is most likely at each state.
%    Before training (first row), each option is randomly initialized, but after training (second row) each option specializes in a specific skill (a navigation pattern). In this example, the demonstrated trajectory can be generated with high probability by using option $3$ and $2$. The last row shows a heat-map indicating where each options is likely to be called by the policy learned in the new task. As can be seen, the agent learns to use each option in very specific situations. For example, option $1$ is likely to be called to make the agent move up, if it is located in one of the bottom rooms.}
%    \label{fig:options}
%    \vspace{-3mm}
%\end{figure*}

To understand the reason behind the improvement in performance resulting from offline option learning, we turn the reader's attention to Figure \ref{fig:options}. The figure is a visualization of the policy learned by the agent on a particular task: navigate from a specific location in the bottom-left room to a location in the top-right room in a small ``four-room'' domain of size $10\times15$. \footnote{We show a smaller domain than used in the experiments for ease of visualization}

The new task to solve is shown in the top-left figure, while the solution found is shown in the top-right figure. Each of the remaining rows of images shows how each option was learned and used in the new task.
The first row show the options learned after training; the highlighted path depicts one of the sample trajectories used for training, the colors correspond to the probability that the options would take the demonstrated action, and the arrows indicate the most likely action to be taken by the option.

The middle row depicts a heatmap indicating how each option was used to solve this specific task. It shows the probability that $\pi$ would execute each option at any given state. Finally, the last row depicts a heatmap indicating the probability of termination for each option given the state. 

Looking at the learned options from these different perspectives provides some insight into how they are being exploited. For example, option $1$ is generally useful to navigate towards the top rooms and, since the goal in this task is in the top-right room, the option is mainly called in the bottom rooms. Also notice that the option is likely to terminate in the top left and bottom right rooms in the states that would lead the agent to get ``trapped'' against a wall. These options, when used in combination in specific regions, allow the agent to efficiently tackle problems it has not encountered before.

\subsection{Experiments using Atari 2600 Games}
We evaluated the quality of the options learned by our framework in two different Atari 2600 games: Breakout and Amidar.
We trained the policy over options using A3C \cite{a3c} with grayscale images as inputs. Options were represented by a two layer convolutional neural network, and were given the previous two frames as input. For each task variation we randomly picked an integer between $1$ and $20$, and let the agent act randomly for that number of time-steps before learning (changing the initial state distribution), we sampled a number in the range $(0,10]$ and use it to scale the reward the agent received (changing the reward function), and allowed for a number of frames to be skipped after taking each action (changing the transition function).
For used three different tasks for training for each game, and sampled $12$ trajectories for training; we used five new tasks for testing. Each trajectory lasted until a life was lost, not for the entire duration of the episode. The options were represented by a two-layer neural network, where the input was represented by gray scale images of the last two frames. We ran $32$ training agents in parallel on CPUs, the learning rate was set to $0.0001$ and the discount factor $\gamma$ was set to $0.99$.

Figures \ref{fig:result-breakout} and \ref{fig:result-amidar} show the performance of the agent as a function of training time in Breakout and Amidar, respectively. The plots show that given good choices of hyperparameters, the learned options led to a clear improvement in performance during training.
For both domains, we found that $\lambda_2=5,000$ led to a reasonable trade-off between the first two term in $\hat J$, and report results with three different values for the regularization term: $\lambda_1=0.0$ (no regularization), $\lambda_1=0.01$ and $\lambda_1=0.1$. 
% In both scenarios, we can see that the options before training led to similar or worse learning performance than simply using primitive actions. On the other hand, the learned options, specially when using some level of regularization, led to a clear improvement in performance during training.
Note that our results do not necessarily show that the options result in a better final policy, but they improve exploration early in training and enable the agent to learn more effectively.
% Note that this does not necessarily mean that the options enable the agent to find a better policy to solve a task, but only that they facilitate the process of finding such policy.

\begin{figure}
    \centering
    \includegraphics[width=1.1\linewidth]{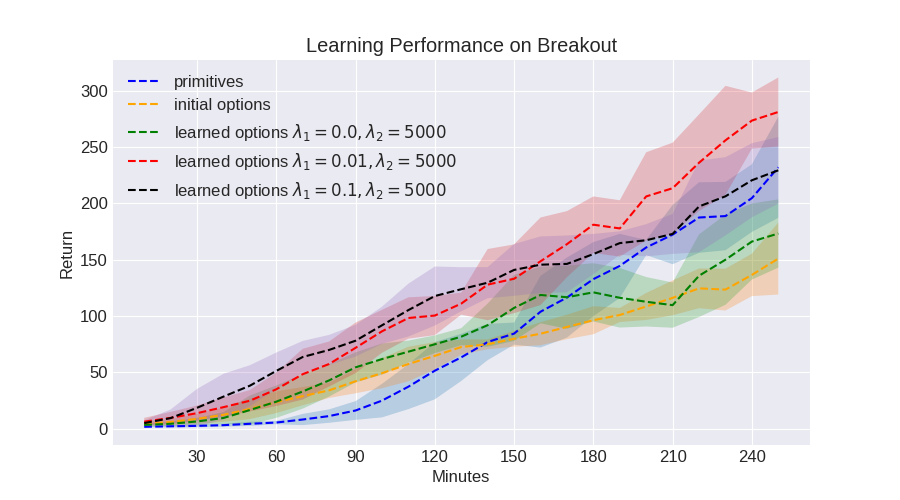}
    \caption{Average returns on Breakout comparing primitives (blue), options before training (orange) and learned options for different values of $\lambda_1$ and $\lambda_2$. The shaded region indicates the standard error.}
    \label{fig:result-breakout}
\end{figure}%
    
\begin{figure}
    \centering
    \includegraphics[width=1.1\linewidth]{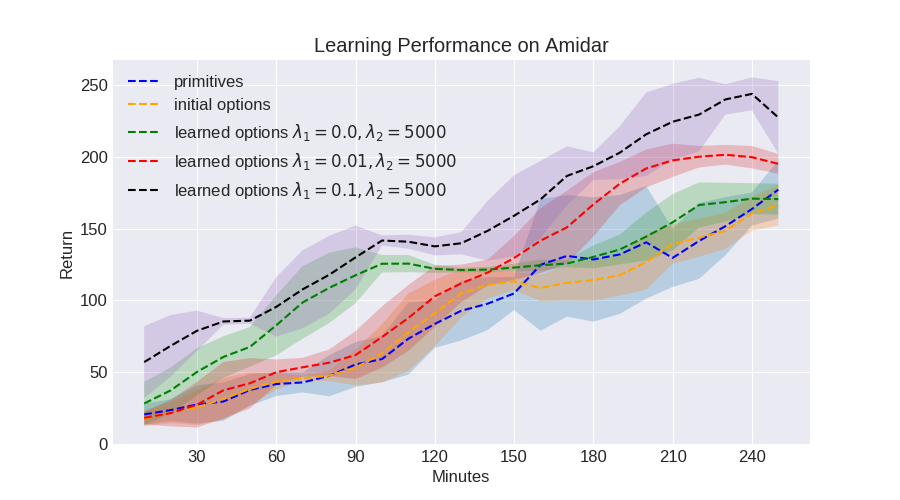}
    \caption{Average returns on Amidar comparing primitives (blue), options before training (orange) and learned options for different values of $\lambda_1$ and $\lambda_2$. The shaded region indicates the standard error.}
    \label{fig:result-amidar}
\end{figure}
    
% \begin{figure}[t!]
%     \centering
%              \captionsetup{justification=centering}
%
%%         % \captionsetup[subfigure]{justification=centering}
%     \begin{subfigure}[t]{0.48\linewidth}
%         \centering
%         \includegraphics[width=\linewidth]{breakout_clean.png}
%         \caption{Average returns on Breakout comparing primitives (blue), options before training (orange) and learned options for different values of $\lambda_1$ and $\lambda_2$. The shaded region indicates the standard error.}
%         \label{fig:result-breakout}
%     \end{subfigure}%
%     
%     
%     \begin{subfigure}[t]{0.48\linewidth}
%         \centering
%         \includegraphics[width=\linewidth]{amidar_clean.png}
%         \caption{Average returns on Amidar comparing primitives (blue), options before training (orange) and learned options for different values of $\lambda_1$ and $\lambda_2$. The shaded region indicates the standard error.}
%         \label{fig:result-amidar}
%     \end{subfigure}
%      \caption{Results on four-room environment. The figure on the left shows the evolution of the training loss as new options are introduced, the figure on the right shows the learning curves on test domains.}
%      \label{fig:maze-results}
%     \vspace{-3mm}
% \end{figure}

\begin{figure}[h]
    \centering
    \includegraphics[width=\linewidth]{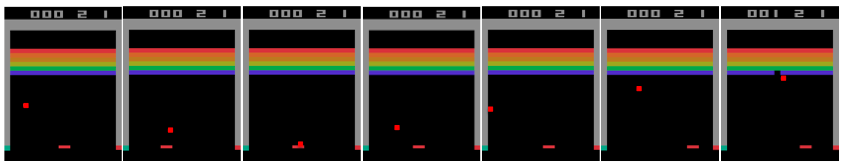}
    \caption{Visualization of a learned option executed until termination on Breakout. The option learned to catch the ball bouncing off the left wall and terminates with high probability before the ball bounces a wall again (ball size increased for visualization). }
    \label{fig:option_breakout}
\end{figure}
    
\begin{figure}[h]
    \includegraphics[width=\linewidth]{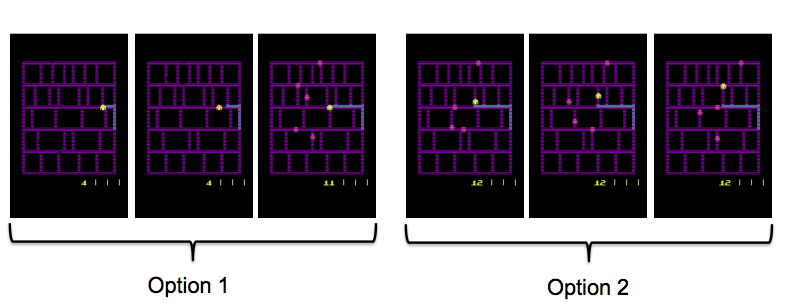}
    \caption{Visualization of two learned options on Amidar. The agent is shown in yellow and enemies in pink. Option 1 learned to move up, at the beginning of the game, and turn left until getting close to an intersection. Option 2 learned to turn in that intersection and move up until reaching the next one. }
    \label{fig:option_amidar}
\end{figure}

Figure \ref{fig:option_breakout} depicts the behavior for one of the learned options on Breakout. 
The option efficiently catches the ball after it bounces off the left wall, and then terminates with high probability before the ball has to be caught again. 
Bear in mind that the option remains active for many time-steps, significantly reducing the number of decisions made by the policy over options.
However, it does not maintain control for so long that the agent is unable to respond to changing circumstances.
Note that the option is only useful in specific case; for example, it was not helpful in returning a ball bounced off the right wall. 
That is to say, the option specialized in a specific sub-task within the larger problem: a highly desirable property for generally useful options.

Figure \ref{fig:option_amidar} shows the selection of two of the options learned for Amidar when starting a new game. 
At the beginning of the game, option 1 is selected, which takes the agent to a specific intersection before terminating.
% (in yellow) up to the first intersection and near another intersection before terminating. 
The agent then selects option 2, which chooses a direction at the intersection, follows the resulting path, and terminates at the next intersection. Note that the agent does not need to repeatedly select primitive actions in order to simply follow a previously chosen path. Having access to these types of options enables an agent to easily replicate known good behaviors, allowing for faster and more meaningful exploration of the state space.

\section{Conclusion and Future Work}
\label{sec:conclusion}

In this work we presented an optimization objective for learning options from demonstrations obtained from learned policies on a set of tasks.
Optimizing the objective results in a set of options that allows an agent to reproduce the behavior while minimizing the number of decisions made by the policy over options, which are able to improve the learning ability of the agent on new tasks.

There are some clear directions for future development. 
While we have shown that our method is capable of discovering powerful options, properly tuning the hyperparameters, $\lambda_1$ and $\lambda_2$, is necessary for learning appropriate options. 
In complex environments, this is not an easy task. 
Future work could study methods for finding the right balance between hyperparameters automatically or, if possible, eliminate the need for such hyperparameters altogether. 
Another possible dimension of improvement is to study how to extend the proposed ideas to the online setting; an agent may be able to sample trajectories as it is learning a task and progressively use them to continuously improve its option set. 
% We assume that trajectories are obtained from (near)-optimal policies and learn options that allow an agent to reproduce those trajectories while minimizing (in expectation) the number of decisions the agent has to make in order to do so. 

We provided results showing how options adapt to the trajectories provided and showed, through several experiments, that the identified options are capable of significantly improving the learning ability of an agent.
The resulting options encode meaningful abstractions that help the agent interact with and learn from its environment more efficiently.

\bibliographystyle{named}
\bibliography{ijcai20}

\appendix

\clearpage

\onecolumn

\section{Appendix}

The following list defines the notation used in all derivations:

\begin{enumerate}
    \item $A_t$: random variable denoting action taken at step $t$.
    \item $S_t$: random variable denoting state at step $t$.
    \item $H_t$: random variable denoting history up to step $t$. $H_t = (S_0, A_0, S_1, A_1, \dots, S_t)$.
    \item $T_t$: random variable denoting the event that the option used at step $t-1$ terminates at state $S_t$. 
    \item $\pi$: policy over options.
    \item $P$: transition function. $P(s,a,s')$ denotes the probability of transitioning to state $s'$ by taking action $a$ in state $s$  
    \item $O_t$: random variable denoting the option selected for execution at state $S_t$.
    \item $o$: option defined as $o = (\mu_o, \beta_o)$, where $\mu_o$ is the option policy for option and $\beta_o$ is the termination function.
    \item Assume primitives are options that perform only 1 action and last for 1 time-step.
    \item $\mathcal{O}$: set of available options.
    % \item For convenience define $\gamma = \pi, \mu_1, \beta_1, \dots, \mu_n, \beta_n$.
\end{enumerate}

We can compute the probability of an option terminating at state $s_t$ and generating a trajectory $h_t$ as:
\begin{equation}
\begin{aligned}
\Pr(T_t=1,H_t=h_t|\pi, \mathcal{O})  =
\Pr(T_t=1| H_t=h_t,\pi, \mathcal{O})  \Pr(H_t=h_t |\pi, \mathcal{O}) 
\end{aligned}
\end{equation}

To compute the proposed objective $J$ we need to find an expression for $\Pr(T_t=1| H_t=h_t,\pi, \mathcal{O})$ and \linebreak $\Pr(H_t=h_t |\pi, \mathcal{O})$ in terms of known quantities.

\subsection{Appendix A - Derivation of $\hat J$}

Recall $J(\pi, \mathcal{O}, H) = \mathbf{E}\left[ \sum_{t=1}^{|h|} \Pr(T_t=0,H_t| \pi, \mathcal{O}) \right]$, ignoring the regularization term. Assuming access to a set $\mathcal H$ of sample trajectories, we start by estimating $J$ from sample averages and derive the objective $\hat J$ as follows:

$
\begin{aligned}
J(\pi, \mathcal{O}, \mathcal{H}) \approx & \frac{1}{|\mathcal H|} \sum_{h \in \mathcal H} \sum_{t=1}^{|h|} \Pr(T_t=0,H_t=h_t|\pi, \mathcal{O})  \\
=&\frac{1}{|\mathcal H|} \sum_{h \in \mathcal H} \sum_{t=1}^{|h|}
\big(1- \Pr(T_t=1| H_t=h_t,\pi, \mathcal{O}) \big)  \Pr(H_t=h_t |\pi, \mathcal{O}) \\
=&\frac{1}{|\mathcal H|} \sum_{h \in \mathcal H} \sum_{t=1}^{|h|}
\bigg(1 - \mathbf{E} \big[T_t| H_t=h_t,\pi, \mathcal{O} \big]\bigg)  \Pr(H_t=h_t |\pi, \mathcal{O}) 
\end{aligned}
$

It can easily be seen that to maximize the above expression $\mathbf{E} \big[T_t| H_t=h_t,\pi, \mathcal{O} \big]$ should be minimized while $\Pr(H=h |\pi, \mathcal{O})$ should be maximized. Given that for long trajectories the expected number of terminations increases while the probability of generating the trajectories goes to $0$, we normalize the number of terminations by the lenght of the trajectory, $|h|$, and adjust a hyperparameter, $\lambda_2$, to prevent one term from dominating the other during optimization. Based on this observation we propose optimizing the following objective:
\begin{equation}
\begin{aligned}
\hat{J}(\pi, \mathcal{O}, \mathcal{H}) =  \frac{1}{\mathcal{H}} \sum_{h_{|h|} \in \mathcal H} \lambda_2 \Pr(H=h | \pi, \mathcal{O}) -  \frac{ \sum_{t=1}^{|h|} \mathbf{E} \left[ T_t \middle | H_t=h_t, \pi, \mathcal{O} \right] }{|h|}. \nonumber
\end{aligned}
\end{equation}

This objective allow us to control a trade-off, through $\lambda_2$, of how much we care about the options reproducing the demonstrated trajectories vs. how much we want the agent to minimize the number of decisions.

\subsection{Appendix B - Proof of Theorem 1}

\textbf{Theorem 1}
\textit{Given a set of options $\mathcal{O}$ and a policy $\pi$ over options, the expected number of terminations for a trajectory $h$ of length $|h|$ is given by:}

\begin{align*}
\sum_{t=1}^{|h|} \mathbf{E} \left[ T_t=1 \middle | H_t=h_t, \pi, \mathcal{O} \right]  
    =& \sum_{t=1}^{|h|} \bigg( \sum_{o \in \mathcal O} \beta_o(s_t) \frac{\mu_o(s_{t-1},a_{t-1}) \Pr(O_{t-1}=o|H_{t-1}=h_{t-1},\pi, \mathcal{O})}{ \sum_{o' \in \mathcal O} \mu_{o'}(s_{t-1},a_{t-1}) \Pr(O_{t-1}=o'|H_{t-1}=h_{t-1},\pi, \mathcal{O})} \bigg) ,  \nonumber
\end{align*}

\textit{where},
\begin{align}
\Pr(O_{t-1}=o|H_{t-1}=h_{t-1},\pi, \mathcal{O}) = \bigg[ \bigg( \pi(s_{t-1}, o) \beta_o(s_{t-1}) \bigg)  \bigg( \nonumber P(s_{t-2},a_{t-2},s_{t-1}) \mu_o(s_{t-2},a_{t-2}) \\ \nonumber
\times \Pr(O_{t-2}=o|H_{t-2}=h_{t-2},\pi, \mathcal{O}) (1-\beta_o(s_{t-1})) \bigg) \bigg],
\end{align}
and 
$\Pr(O_0=o|H_0=h_0,\pi, \mathcal{O}) = \pi(s_0,o)$.
\begin{proof}
Notice that $\sum_{t=1}^{|h|} \mathbf{E} \left[ T_t=1 \middle | H_t=h_t, \pi, \mathcal{O} \right] = \sum_{t=1}^{|h|} \Pr(T_t=1|H_t=h_t, \pi, \mathcal{O}) \; 1$, so if we find an expression for $\Pr(T_t=1|H_t=h_t, \pi, \mathcal{O})$, we can calculate the expectation exactly. 
We define $\Pr(T_0=1|H_1=h_1,\pi, \mathcal{O})=1$ for ease of derivation even though there is no option to terminate at $T_0$.

\begin{align*}
    \Pr(T_t = 1| H_t = h_t, \pi, \mathcal{O}) =& \sum_{o \in \mathcal O} \Pr(T_t = 1 | O_{t-1} = o, H_t = h_t, \pi, \mathcal{O}) \Pr(O_{t-1} = o|H_t = h_t, \pi, \mathcal{O}) \\
    =& \sum_{o \in \mathcal O} \beta_o(s_t) \Pr(O_{t-1} = o|H_t = h_t, \pi, \mathcal{O}) \\
    =& \sum_{o \in \mathcal O} \beta_o(s_t) \Pr(O_{t-1}=o|H_{t-1}=h_{t-1},A_{t-1}=a_{t-1},S_t=s_t,\pi, \mathcal{O}) \\
    =& \sum_{o \in \mathcal O} \beta_o(s_t) \frac{\Pr(S_t=s_t|H_{t-1}=h_{t-1},A_{t-1}=a_{t-1},O_{t-1}=o,\pi, \mathcal{O}) }{\Pr(S_t=s_t|H_{t-1}=h_{t-1},A_{t-1}=a_{t-1},\pi, \mathcal{O})} \\ 
    &\times \Pr(O_{t-1}=o|H_{t-1}=h_{t-1},A_{t-1}=a_{t-1},\pi, \mathcal{O}) \\
    =& \sum_{o \in \mathcal O} \beta_o(s_t)  \frac{\Pr(S_t=s_t|H_{t-1}=h_{t-1},A_{t-1}=a_{t-1},\pi, \mathcal{O})}{\Pr(S_t=s_t|H_{t-1}=h_{t-1},A_{t-1}=a_{t-1},\pi, \mathcal{O})} \\
    &\times \Pr(O_{t-1}=o|H_{t-1}=h_{t-1},A_{t-1}=a_{t-1},\pi, \mathcal{O})\\
    =& \sum_{o \in \mathcal O} \beta_o(s_t) \Pr(O_{t-1}=o|H_{t-1}=h_{t-1},A_{t-1}=a_{t-1},\pi, \mathcal{O}) \\
    =& \sum_{o \in \mathcal O} \beta_o(s_t) \frac{\Pr(A_{t-1}=a_{t-1}|H_{t-1}=h_{t-1},O_{t-1}=o,\pi, \mathcal{O}) \Pr(O_{t-1}=o|H_{t-1}=h_{t-1},\pi, \mathcal{O})}{\Pr(A_{t-1}=a_{t-1}|H_{t-1}=h_{t-1},\pi, \mathcal{O})} \\
    =& \sum_{o \in \mathcal O} \beta_o(s_t) \frac{\mu_o(s_{t-1},a_{t-1}) \Pr(O_{t-1}=o|H_{t-1}=h_{t-1},\pi, \mathcal{O})}{\Pr(A_{t-1}=a_{t-1}|H_{t-1}=h_{t-1},\pi, \mathcal{O})} \\
    \displaybreak\\
    =& \sum_{o \in \mathcal O} \beta_o(s_t) \frac{\mu_o(s_{t-1},a_{t-1}) \Pr(O_{t-1}=o|H_{t-1}=h_{t-1},\pi, \mathcal{O})}{ \sum_{o' \in \mathcal O} \Pr(A_{t-1}=a_{t-1},O_{t-1}=o'|H_{t-1}=h_{t-1},\pi, \mathcal{O})} \\
    =& \sum_{o \in \mathcal O} \beta_o(s_t) \mu_o(s_{t-1},a_{t-1}) \Pr(O_{t-1}=o|H_{t-1}=h_{t-1},\pi, \mathcal{O}) \\
    &\times \big( \sum_{o' \in \mathcal O} \Pr(A_{t-1}=a_{t-1}|O_{t-1}=o',H_{t-1}=h_{t-1},\pi, \mathcal{O}) \\
    &\times \Pr(O_{t-1}=o'|H_{t-1}=h_{t-1},\pi, \mathcal{O}) \big)^{-1} \\
    % =& \sum_{o \in \mathcal O} \beta_o(s_t) \frac{\mu_o(s_{t-1},a_{t-1}) \Pr(O_{t-1}=o|H_{t-1}=h_{t-1},\pi, \mathcal{O})}{ \sum_{o' \in \mathcal O} \Pr(A_{t-1}=a_{t-1}|O_{t-1}=o',H_{t-1}=h_{t-1},\pi, \mathcal{O})\Pr(O_{t-1}=o'|H_{t-1}=h_{t-1},\pi, \mathcal{O})} \\
    =& \sum_{o \in \mathcal O} \beta_o(s_t) \frac{\mu_o(s_{t-1},a_{t-1}) \Pr(O_{t-1}=o|H_{t-1}=h_{t-1},\pi, \mathcal{O})}{ \sum_{o' \in \mathcal O} \mu_{o'}(s_{t-1},a_{t-1}) \Pr(O_{t-1}=o'|H_{t-1}=h_{t-1},\pi, \mathcal{O})}
\end{align*}

We are left with finding an expression in terms of known probabilities for $\Pr(O_{t-1}=o|H_{t-1}=h_{t-1},\pi, \mathcal{O})$.

\begin{align*}
    \Pr(O_{t-1}=o|H_{t-1}=h_{t-1},\pi, \mathcal{O})=&  \big[ \Pr(O_{t-1}=o, T_{t-1}=1|H_{t-1}=h_{t-1},\pi, \mathcal{O}) \\
    &+ \Pr(O_{t-1}=o, T_{t-1}=0|H_{t-1}=h_{t-1},\pi, \mathcal{O}) \big] \\
    =& \bigg[ \big( \Pr(O_{t-1}=o |H_{t-1}=h_{t-1},T_{t-1}=1,\pi, \mathcal{O}) \\
    &\times \Pr(T_{t-1}=1|H_{t-1}=h_{t-1},\pi, \mathcal{O}) \big)\\
    &+ \big( \Pr(O_{t-1}=o|H_{t-1}=h_{t-1}, T_{t-1}=0,\pi, \mathcal{O}) \\
    &\times (1-\Pr(T_{t-1}=1|H_{t-1}=h_{t-1},\pi, \mathcal{O})) \big) \bigg] \\
    =& \bigg[ \big( \pi(s_{t-1}, o) \Pr(T_{t-1}=1|H_{t-1}=h_{t-1},\pi, \mathcal{O}) \big)\\
    &+ \big( \Pr(O_{t-1}=o|H_{t-1}=h_{t-1}, T_{t-1}=0,\pi, \mathcal{O}) \\
    &\times (1-\Pr(T_{t-1}=1|H_{t-1}=h_{t-1},\pi, \mathcal{O})) \big) \bigg] \\
    =& \bigg[ \big( \pi(s_{t-1}, o) \beta_o(s_{t-1}) \big) + \\
    &\times \big( \Pr(O_{t-1}=o|H_{t-1}=h_{t-1}, T_{t-1}=0,\pi, \mathcal{O}) (1-\beta_o(s_{t-1})) \big) \bigg] \\
\end{align*}

Given that by convention, $\Pr(T_0=1|H_0=h_0,\pi, \mathcal{O})=1.0$, we are now left with figuring out how to calculate $\Pr(O_{t-1}=o|H_{t-1}=h_{t-1}, T_{t-1}=0,\pi, \mathcal{O})$

\begin{align*}
\Pr(O_{t-1}=o|H_{t-1}=h_{t-1}, T_{t-1}=0,\pi, \mathcal{O}) =& \Pr(O_{t-2}=o,A_{t-2}=a_{t-2},S_{t-1}=s_{t-1}|H_{t-1}=h_{t-1},\pi, \mathcal{O}) \\
    =& \Pr(A_{t-2}=a_{t-2},S_{t-1}=s_{t-1}|O_{t-2}=o,H_{t-1}=h_{t-1},\pi, \mathcal{O}) \\ &\times \Pr(O_{t-2}=o|H_{t-1}=h_{t-1},\pi, \mathcal{O}) \\
    =& \Pr(S_{t-1}=s_{t-1}|A_{t-2}=a_{t-2},O_{t-2}=o,H_{t-1}=h_{t-1},\pi, \mathcal{O}) \\
    &\times \Pr(A_{t-2}=a_{t-2}|O_{t-2}=o,H_{t-1}=h_{t-1},\pi, \mathcal{O}) \\ &\times \Pr(O_{t-2}=o|H_{t-1}=h_{t-1},\pi, \mathcal{O}) \\
    =& P(s_{t-2},a_{t-2},s_{t-1}) \mu_o(s_{t-2},a_{t-2}) \Pr(O_{t-2}=o|H_{t-1}=h_{t-1},\pi, \mathcal{O}) \\
    =& P(s_{t-2},a_{t-2},s_{t-1}) \mu_o(s_{t-2},a_{t-2}) \Pr(O_{t-2}=o|H_{t-2}=h_{t-2},\pi, \mathcal{O}) \\
\end{align*}

where $\Pr(O_0=o|H_0=h_0,\pi, \mathcal{O}) = \pi(s_0,o)$

Using the recursive function $\Pr(O_{t-1}=o'|H_{t-1}=h_{t-1},\pi, \mathcal{O})$, the expected number of terminations for a given trajectory is given by:
$$
\begin{aligned}
\sum_{t=1}^{|h|} \mathbf{E} \left[ T_t=1 \middle | H_t=h_t, \pi, \mathcal{O} \right]  
    =& \sum_{t=1}^{|h|} \bigg( \sum_{o \in \mathcal O} \beta_o(s_t) \frac{\mu_o(s_{t-1},a_{t-1}) \Pr(O_{t-1}=o|H_{t-1}=h_{t-1},\pi, \mathcal{O})}{ \sum_{o' \in \mathcal O} \mu_{o'}(s_{t-1},a_{t-1}) \Pr(O_{t-1}=o'|H_{t-1}=h_{t-1},\pi, \mathcal{O})} \bigg) ,  \nonumber
\end{aligned}
$$
\end{proof}

\subsection{Appendix C - Proof of Theorem 2}

\textbf{Theorem 2}

\textit{Given a set of options $\mathcal{O}$ and a policy $\pi$ over options, the probability of generating a trajectory $h$ of length $|h|$ is given by:}

$\Pr(H_{|h|}=h_{|h|} |\pi, \mathcal{O}) =  d_0(s_0) \bigg[ \sum_{o \in \mathcal{O}} \pi(s_0,o) \mu_o(s_0,a_0) f(h_{|h|},o,1) \bigg] \prod_{k=0}^{|h|-1} P(s_k,a_k,s_{k+1})$, \textit{where $f$ is a recursive function defined as:}

\[
  f(h_t, o, i) = 
  \begin{cases}
    1, & \text{if } i = t \\
    \bigg[ \beta_o(s_i) \sum_{o' \in \mathcal{O}} \pi(s_{i+1},o') \mu_{o'}(s_{i+1}, a_{i+1}) f(h_t,o',i+1)  \\
    + (1-\beta_o(s_i)) \mu_o(s_{i+1},a_{i+1}) f(h_t,o,i+1) \bigg], & \text{otherwise}
  \end{cases}
\]

\begin{proof}

We define $H_{i,t}$ to be the history from time $i$ to time $t$, that is, $H_{i,t} = (S_i, A_i, S_{i+1}, A_{i+1}, \dots, S_t$), where $i < t$. If $i = t$, the history would contain a single state. 

\begin{align*}
    \Pr(H_t = h_t|\pi, \mathcal{O}) =& \Pr(S_0 = s_0|\pi, \mathcal{O}) \Pr(H_{1,t} = h_{1,t}, A_0 = a_0 | S_0 = s_0, \pi, \mathcal{O}) \\
    =& d_0(s_0) \Pr(H_{1,t} = h_{1,t}, A_0 = a_0 | S_0 = s_0, \pi, \mathcal{O}) \\
    =& d_0(s_0) \sum_{o \in \mathcal O}  \Pr(H_{1,t} = h_{1,t}, A_0 = a_0, O_o=o | S_0 = s_0, \pi, \mathcal{O}) \\
    =& d_0(s_0) \sum_{o \in \mathcal O}  \Pr(O_0 = o | S_0 = s_0, \pi, \mathcal{O}) \Pr(H_{1,t} = h_{1,t}, A_0 = a_0 | S_0 = s_0, O_0 = o, \pi, \mathcal{O}) \\
    =& d_0(s_0) \sum_{o \in \mathcal O}  \pi_(s_0, o) \Pr(H_{1,t} = h_{1,t}, A_0 = a_0 | S_0 = s_0, O_0 = o, \pi, \mathcal{O}) \\
    =& d_0(s_0) \sum_{o \in \mathcal O}  \pi_(s_0, o) \Pr(A_0 = a_o | S_0 = s_0, O_0 = o, \pi, \mathcal{O}) \\
    &\times \Pr(H_{1,t} = h_{1,t} | S_0 = s_0, O_0 = o, A_0 = a_0, \pi, \mathcal{O}) \\
    =& d_0(s_0) \sum_{o \in \mathcal O}  \pi_(s_0, o) \mu_o(s_0, a_o) \Pr(H_{1,t} = h_{1,t} | S_0 = s_0, O_0 = o, A_0 = a_0, \pi, \mathcal{O}).
\end{align*}

We now need to find an expression to calculate $\Pr(H_{1,t} = h_{1,t} | S_0 = s_0, O_0 = o, A_0 = a_0, \pi, \mathcal{O})$. Consider the probability of seeing history $h_{i, t}$ given the previous state, $s$, the previous option, $o$, and the previous action, $a$:

\begin{align*}
     &\Pr(H_{i, t} = h_{i, t} | S_{i-1} = s, O_{i - 1} = o, A_{i-1} = a) \\
     =& \Pr(S_i = s_i | S_{i-1} = s, O_{i - 1} = o, A_{i-1} = a) \Pr(H_{i+1, t} = h_{i+1,t}, A_i = a_i | S_{i-1}=s, O_{i - 1} = o, A_{i-1} = a, S_i = s_i) \\
    =& P(s, a, s_i) \Pr(H_{i+1, t} = h_{i+1,t}, A_i = a_i | S_{i-1}=s, O_{i - 1} = o, A_{i-1} = a, S_i = s_i) \\
    =& P(s, a, s_i) \Pr(H_{i+1, t} = h_{i+1,t}, A_i = a_i | O_{i - 1} = o, A_{i-1} = a, S_i = s_i) \\
    =& P(s, a, s_i) \big [
        \Pr(T_i = 1 | O_{i - 1} = o, A_{i-1} = a, S_i = s_i) \\
        &\times \Pr(H_{i+1, t} = h_{i+1,t}, A_i = a_i | O_{i - 1} = o, A_{i-1} = a, S_i = s_i, T_i = 1) \\
    &+ \Pr(T_i = 0 | O_{i - 1} = o, A_{i-1} = a, S_i = s_i) \\
    &\times \Pr(H_{i+1, t} = h_{i+1,t}, A_i = a_i | O_{i - 1} = o, A_{i-1} = a, S_i = s_i, T_i = 0)
    \big ] \\
    =& P(s, a, s_i) \big [
        \beta_o(s_i) \\
        &\times \Pr(H_{i+1, t} = h_{i+1,t}, A_i = a_i | O_{i - 1} = o, A_{i-1} = a, S_i = s_i, T_i = 1) \\
    &+ (1 - \beta_o(s_i)) \\
    &\times \Pr(H_{i+1, t} = h_{i+1,t}, A_i = a_i | O_{i - 1} = o, A_{i-1} = a, S_i = s_i, T_i = 0)
    \big ].
\end{align*}

Even though the equation above might seem complicated, there are only two cases we need to consider: either the current option terminates and a new one must be selected (the first term), or the current option does not terminate (the second term).
Let's consider each of them separately.

\textbf{Case 1 - option terminates:} If we terminate, we sum over new options:

\begin{align*}
    &\Pr(H_{i+1, t} = h_{i+1,t}, A_i = a_i | O_{i - 1} = o, A_{i-1} = a, S_i = s_i, T_i = 1) \\
    =& \sum_{o' \in \mathcal O} \Pr(O_i = o' | O_{i - 1} = o, A_{i-1} = a, S_i = s_i, T_i = 1) \\ &\times \Pr(H_{i+1, t} = h_{i+1,t}, A_i = a_i | O_{i - 1} = o, A_{i-1} = a, S_i = s_i, T_i = 1, O_i = o') \\
    =& \sum_{o' \in \mathcal O} \pi(s_i, o')  \Pr(H_{i+1, t} = h_{i+1,t}, A_i = a_i | O_{i - 1} = o, A_{i-1} = a, S_i = s_i, T_i = 1, O_i = o') \\
    =& \sum_{o' \in \mathcal O} \pi(s_i, o')  \Pr(H_{i+1, t} = h_{i+1,t}, A_i = a_i | S_i = s_i, O_i = o') \\
    =& \sum_{o' \in \mathcal O} \pi(s_i, o')  \Pr(A_i = a_i | S_i = s_i, O_i = o') \Pr(H_{i+1, t} = h_{i+1,t} | S_i = s_i, O_i = o', A_i = a_i) \\
    =& \sum_{o' \in \mathcal O} \pi(s_i, o')  \mu_{o'}(s_i, a_i) \Pr(H_{i+1, t} = h_{i+1,t} | S_i = s_i, O_i = o', A_i = a_i). \\
\end{align*}

Note that the expanded probability has the same form as $\Pr(H_{i, t} = h_{i, t} | S_{i-1} = s, O_{i - 1} = o, A_{i-1} = a)$.

\textbf{Case 2 - option does not terminate:} 
This tells us that $O_i = o$, so we may drop the dependency on the $i-1$ terms:

\begin{align*}
    &\Pr(H_{i+1, t} = h_{i+1,t}, A_i = a_i | S_{i-1} = s, O_{i - 1} = o, A_{i-1} = a, S_i = s_i, T_i = 0) \\
    =& \Pr(H_{i+1, t} = h_{i+1,t}, A_i = a_i | S_i = s_i, O_i = o) \\
    =& \Pr(A_i = a_i | S_i = s_i, O_i = 0) \Pr(H_{i+1, t} = h_{i+1,t} | S_i = s_i, O_i = o, A_i = a_i) \\
    =& \mu_o(s_i, a_i) \Pr(H_{i+1, t} = h_{i+1,t} | S_i = s_i, O_i = o, A_i = a_i).
\end{align*}

Plugging these two cases back into our earlier equation yields:

\begin{align*}
     &\Pr(H_{i, t} = h_{i, t} | S_{i-1} = s, O_{i - 1} = o, A_{i-1} = a) \\ 
    =& P(s, a, s_i) \big [
        \beta_o(s_i) \sum_{o' \in \mathcal O} \pi(s_i, o')  \mu_{o'}(s_i, a_i) \Pr(H_{i+1, t} = h_{i+1,t} | S_i = s_i, O_i = o', A_i = a_i) \\
    &+ (1 - \beta_o(s_i)) \mu_o(s_i, a_i) \Pr(H_{i+1, t} = h_{i+1,t} | S_i = s_i, O_i = o, A_i = a_i) \big].
\end{align*}

Note that each term contains an expression of the same form, $\Pr(H_{i, t} = h_{i, t} | S_{i-1} = s, O_{i - 1} = o, A_{i-1} = a)$.
We can therefore compute the probability recursively. 
Our recursion will terminate when we consider $i = t$, as $H_{t, t}$ contains a single state, and we adopt the convention of its probability to be 1.
Notice that for every recursive step, both inner terms will produce a $P(s, a, s_i)$ term.
Consider the result when we factor every recursive $P(s, a, s_i)$ term to the front of the equation.
We define the following recursive function:
\[
  f(h_t, o, i) = 
  \begin{cases}
    1, & \text{if } i = t \\
    \bigg[ \beta_o(s_i) \sum_{o' \in \mathcal{O}} \pi(s_{i+1},o') \mu_{o'}(s_{i+1}, a_{i+1}) f(h_t,o',i+1)  \\
    + (1-\beta_o(s_i)) \mu_o(s_{i+1},a_{i+1}) f(h_t,o,i+1) \bigg], & \text{otherwise}
  \end{cases}
\].

Notice that this is the recursive probability described above, but with the $P(s, a, s')$ terms factored out.
We now see that:

\begin{align*}
     &\Pr(H_{i, t} = h_{i, t} | S_{i-1} = s_{i-1}, O_{i - 1} = o, A_{i-1} = a_{i - 1}) = f(h_t, o, i) \prod_{k = i - 1}^{t-1} P(s_k, a_k, s_{k + 1}).
\end{align*}

Plugging this all the back into our original equation for $\Pr(H_t = h_t | \pi, \mathcal{O})$ gives us the desired result:

\begin{align*}
    \Pr(H_{|h|}=h_{|h|} | \pi, \mathcal{O}) = d_0(s_0) \bigg[ \sum_{o \in \mathcal O}  \pi(s_0, o) \mu_o(s_0, a_o) f(h_{|h|}, o, 1) \bigg] \prod_{k = 0}^{t-1} P(s_k, a_k, s_{k + 1}).
\end{align*}

\end{proof}

\subsection{Appendix D - Empirical Validation of Derived Equations}

To double check the derivation of the proposed objective and make sure the implementation was correct, we conducted a simple empirical test to compared the calculated expected number of decisions in a trajectory and the probability of generating each trajectory for a set of $10$ trajectories on $10$ MDPs. The MDPs are simple chains of $7$ states with different transition functions.
We randomly initialized four options and a policy over options, and estimated the probability of generating each trajectory and the expected number of terminations, for each sampled trajectory, by Montecarlo sampling $10,000$ trials. Table \ref{table:validation} presents results for the $10$ trajectories verifying empirically that the equations were correctly derived and implemented. The table compares the empirical and true probability of generating a given trajectory, $\hat{\Pr}(H| \cdot )$ and $\Pr(H| \cdot )$, respectively, and the empirical and true sum  of expected number of decisions an agent has to make to generate those trajectories, $\sum_{t=1}^{|H|} \hat{\mathbf{E}} \left[ T_t \middle | \cdot \right]$ and $\sum_{t=1}^{|H|} \mathbf{E} \left[ T_t \middle | \cdot \right]$, respectively.

\small
\begin{table}[H]
\begin{center}
\begin{tabular}{ |c|c|c|c|c| } 
\hline
 H &  $\hat{\Pr}(H|\pi,\mathcal{O})$  & $\Pr(H|\pi,\mathcal{O})$ & $\sum_{t=1}^{|H|} \hat{\mathbf{E}} \left[ T_t \middle | H_t, \pi,\mathcal{O} \right]$ & $\sum_{t=1}^{|H|} \mathbf{E} \left[ T_t \middle | H_t, \pi,\mathcal{O} \right]$ \\
\hline
$h_1$ & $0.0932$ & $0.0957$ & $3.060$ & $3.178$ \\
\hline
$h_2$ & $0.0158$ & $0.0173$ & $4.139$ & $4.154$ \\
\hline
$h_3$ & $0.2149$ & $0.2122$ & $1.965$ & $2.178$    \\
\hline
$h_4$ & $0.0995$ & $0.0957$ & $2.979$ & $3.178$    \\
\hline
$h_5$ & $0.0962$ & $0.0957$ & $3.024$ & $3.178$    \\
\hline
$h_6$ & $0.1354$ & $0.1384$ & $2.9579$ & $3.1596$    \\
\hline
$h_7$ & $0.00040$ & $0.00038$ & $9.750$ & $8.794$    \\
\hline
$h_8$ & $0.1854$ & $0.1881$ & $2.820$ & $3.072$    \\
\hline
$h_9$ & $0.0379$ & $0.0368$ & $4.2612$ & $4.4790$    \\
\hline
$h_{10}$ & $0.1864$ & $0.1881$ & $2.8404$ & $3.0723$    \\
\hline
\end{tabular}
\end{center}
\caption{Validation of equations and implementation.}
\label{table:validation}
\end{table}
\normalsize

Note that the cases with largest discrepancy between the estimated and calculated number of terminations occur when the probability of generating a trajectory is low. This happens because, since the trajectory is unlikely to be generated, the Monte Carlo sampling is not able to produce enough samples of the trajectory.

%\subsection{Appendix E - Implementation Details for Atari Experiments}
%
%For these experiments we first learned a good performing policy with A3C for each game and sampled $12$ trajectories for training. Each trajectory lasted until a life was lost, not for the entire duration of the episode. The options were represented by a two-layer neural network, where the input was represented by gray scale images of the last two frames. We ran $32$ training agents in parallel on CPUs, the learning rate was set to $0.0001$ and the discount factor $\gamma$ was set to $0.99$.
%
%Because the options can only learn from the states observed in the trajectories, it is possible that when using them, they will be executed in previously unseen states. When this happens, the termination function may decide to never terminate, as it has not seen that region of the state space before. To address this issue, we add a value of $0.05$ to the predicted probability of termination per time-step that the option has been running since executed. Therefore, in our experiments an option cannot run for more than $20$ time-steps in total.  

\end{document}